\documentclass[12pt]{amsart}
\usepackage{hyperref}
\usepackage[margin=1.14in]{geometry} 
\usepackage{lipsum}
\usepackage{amsmath,amsthm,amssymb}
\usepackage{setspace}
\usepackage{algorithm}
\usepackage{graphicx}
\usepackage[numbers,round]{natbib}
 \usepackage{thmtools,thm-restate}
\usepackage{natbib}
\usepackage{amsmath}
\usepackage{amssymb}
\usepackage{amsfonts}
\usepackage{amsthm}
\usepackage{booktabs}
\usepackage{algorithm}
\usepackage{algpseudocode}   
\usepackage{algpseudocode}
\usepackage{paralist}       
\usepackage{float}          
\usepackage{stfloats}
\usepackage{paralist}       
\usepackage{float}          
\usepackage{stfloats}

\newtheorem{thm}{Theorem}
\newtheorem*{thm*}{Theorem}
\newtheorem{lem}[thm]{Lemma}
\newtheorem*{lem*}{Lemma}
\newtheorem{deff}[thm]{Definition}

\newtheorem{rmk}{Remark}

\newtheorem{expl}[thm]{Example}

\newtheorem{cor}[thm]{Corollary}
\newtheorem*{cor*}{Corollary}
\newtheorem{prob*}{Problem}

\newcommand{\R}{\mathbb{R}}

\newcommand{\Norm}[1]{\left\Vert #1 \right\Vert}

\date{}

\hypersetup{
    colorlinks=true,
    linkcolor=blue,
    filecolor=magenta,      
    urlcolor=cyan,
}

\begin{document}
\title[Parameter-Efficient Generative Modeling with Controlled Vector Fields]{Parameter-Efficient Generative Modeling with Controlled Vector Fields}
\author{Peyman Morteza}
\address{Department of Computer Sciences, University of Wisconsin, Madison, 
WI, 53706}
\email{peyman@cs.wisc.edu}
\begin{abstract}
We introduce a continuous-time generative modeling framework, motivated by the
Chow--Rashevskii theorem, that builds expressive flows from a small set of fixed
vector fields and learned scalar controls. Instead of learning an unconstrained
high-dimensional vector field, our framework constructs the velocity by
modulating fixed vector fields with learned scalar control functions. When the
fixed fields are bracket-generating, their Lie algebra spans the ambient space,
providing a mechanism for expressive transport with only a small number of
learned control channels and offering a parameter-efficient geometric
alternative to standard vector-field parameterizations. This decoupled formulation yields a structured and interpretable generative
model in which the number of learned scalar output channels can be chosen
independently of the ambient dimension. We formulate an expressivity principle showing
that, under suitable controllability and well-posedness assumptions, such
controlled flows can transport a source distribution to a target distribution.
We train the resulting model using a continuous-normalizing-flow likelihood
objective and present proof-of-concept experiments on synthetic distributions.
\end{abstract}
\maketitle
\setcounter{tocdepth}{1}
 \tableofcontents
\vspace{-10mm}
\section{Introduction}
Generative models form a cornerstone of modern machine learning and AI, driving advances in areas such as image synthesis, language modeling, and structured data generation \citep{brown2020language,ho2020denoising}. Among the most successful approaches are flow-based and diffusion models, which have shown remarkable capability in capturing complex, high-dimensional data distributions \citep{dinh2016density,song2020score,papamakarios2021normalizing,lipman2022flow}.

In this note, we investigate a continuous-time generative modeling framework
that builds expressive flows from a small set of fixed vector fields and
learned scalar controls. The key idea is to replace an unconstrained
high-dimensional velocity field with a controlled dynamical system whose
underlying vector fields are fixed and whose scalar coefficients are learned. Unlike conventional flow-based models that directly learn high-dimensional vector-valued fields, our approach learns only scalar coefficients over predefined geometric directions (see Lemma \ref{lem:main}). Leveraging the Chow–Rashevskii theorem, we formulate an expressivity principle showing that, under suitable
admissibility, controllability, and well-posedness assumptions, such controlled
flows can transport a source distribution to a target distribution (see Corollary \ref{cor:push}), even when restricted to a limited number of control directions.

Intuitively, the Lie algebra generated by the fixed vector fields spans the full tangent space, enabling the system to reach any state through compositions of a few base directions. This results in a compact, parameter-efficient, and structurally interpretable
generative framework: the fixed vector fields provide the underlying geometric
structure, while the learned scalar controls modulate the dynamics with a small
number of output channels. By structural interpretability, we mean that the generative process is explicitly constructed from a fixed, known set of vector fields, while the learned scalar control functions modulate these directions over time. We instantiate the framework with lightweight MLPs that parameterize the scalar control functions (see Algorithm~\ref{alg:cnf-chowflow}) and evaluate its performance on standard synthetic density benchmarks.

The remainder of this note is organized as follows.
Section~\ref{sec:background} recalls the necessary preliminaries on Lie brackets,
the Chow--Rashevskii theorem, and control systems. Section~\ref{sec:chowflow}
introduces the ChowFlow framework, which replaces unconstrained vector-field
parameterizations with learned scalar controls over fixed vector fields satisfying
a bracket-generating condition. Section~\ref{sec:experiments} presents
proof-of-concept experiments on synthetic distributions. Section~\ref{sec:proofs}
provides detailed proofs of the theoretical results, and Section~\ref{sec:exper} provides additional
experimental details.

\section{Preliminaries}

\label{sec:background}
We provide a brief overview of Lie brackets, the Chow–Rashevskii theorem, and control systems. Detailed expositions can be found in standard references, e.g., \cite{warner1983foundations,jurdjevic1997geometric,montgomery2002tour}.

\begin{deff}[Lie Bracket]
The Lie bracket of two smooth vector fields \( X \) and \( Y \) on a smooth manifold \( M \), denoted \( [X, Y] \), is a smooth vector field defined by,
\[
[X, Y](f) = X(Y(f)) - Y(X(f)),
\]
for any smooth function \( f: M \to \mathbb{R} \), where \( X(f) \) and \( Y(f) \) denote the action of the vector fields on \( f \). 
\end{deff}
\begin{deff}[Left-Invariant Vector Field]
Let \( G \) be a Lie group, and let \( L_g: G \to G \) denote left multiplication by \( g \in G \), i.e., \( L_g(h) = gh \). A smooth vector field \( X \in \mathfrak{X}(G) \) is called \emph{left-invariant} if it is preserved under all left translations,
\[
(dL_g)_h (X(h)) = X(gh), \quad \text{for all } g,h \in G.
\]
Equivalently, \( X \) is completely determined by its value at the identity element \( e \in G \), and extended to all of \( G \) by left translation,
\[
X(g) = (dL_g)_e (X(e)).
\]
The space of left-invariant vector fields is naturally identified with the Lie algebra $T_{e}G$ of $G$.
\end{deff}
\begin{deff}[Control System]
A control system on a smooth manifold \( M \) is a dynamical system of the form,
\[
\dot{x}(t) = \sum_{i=1}^k u_i(t, x(t)) X_i(x(t)),
\]
where \( \{X_1, X_2, \dots, X_k\} \) are smooth vector fields on \( M \), and \( u_i: [0, T] \times M \to \mathbb{R} \) , $1\le i\le k$, are scalar control inputs, or feedback control laws, depending on time and
state. Here, $k$ denotes the number of control channels. When classical ODE solutions are considered,
we assume the \(u_i\) are sufficiently regular to ensure well-posedness.
\end{deff}

\begin{thm}[Chow–Rashevskii Theorem {\citep{chow1940systeme,rashevsky1938connecting}}]
\label{thm:chow}
Let $M$ be a connected smooth manifold, and let $\Delta = \mathrm{span}\{X_1,\dots,X_k\}$ be a distribution generated by smooth vector fields. If the Lie algebra generated by $\{X_1,\dots,X_k\}$ spans the entire tangent space $T_xM$ at every $x \in M$, i.e.,
\[
\mathrm{Lie}\{X_1,\dots,X_k\}(x) = T_x M \qquad \forall x \in M,
\]
then any two points $p,q \in M$ can be connected by a piecewise smooth, absolutely continuous curve $\gamma : [0,1] \to M$ satisfying,
\[
\gamma(0)=p, \qquad \gamma(1)=q, \qquad \dot{\gamma}(t) \in \Delta_{\gamma(t)},
\]
for almost every $t\in[0,1]$.
\end{thm}
\begin{rmk}[Control-theoretic interpretation]
\label{rmk:control_reachability}
The Chow--Rashevskii theorem can be interpreted as a controllability statement:
when the vector fields \(X_1,\ldots,X_k\) are bracket-generating, admissible
curves tangent to their span can connect any two points in the same connected
component of \(M\). In our setting, this motivates the use of a small set of
fixed bracket-generating vector fields as a geometric basis for constructing
continuous-time generative flows.
\end{rmk}
\section{ChowFlow}
\label{sec:chowflow}

In this section, we show how the Chow--Rashevskii theorem motivates a
parameter-efficient class of structured continuous-time generative models. Our key observation is that viewing generative flows through the lens of the
Chow--Rashevskii theorem leads to a parameter-efficient velocity
parameterization: rather than outputting a full \(d\)-dimensional vector field,
the model learns scalar controls over a small set of fixed bracket-generating
vector fields. 
Specifically, expressive generative behavior can be achieved using only a small number of scalar control functions, provided that the underlying vector fields are appropriately chosen and satisfy the bracket-generating condition of Theorem~\ref{thm:chow}. 
This perspective offers a structured way to steer the generative process through a low-dimensional set of learnable coefficients, effectively reusing a shared geometric backbone to produce complex transformations. 
This formulation motivates the construction presented in Algorithm~\ref{alg:cnf-chowflow}.

\subsection{Geometric Control and Controlled Flows}

We begin by establishing the mathematical foundations of the framework. As stated in Theorem~\ref{thm:chow}, under mild regularity conditions, a suitable control system can steer any initial point to an arbitrary target point. 
Let $\mu_0$ and $\mu_1$ be two probability measures on $(M,g)$, and let $\pi \in \Pi(\mu_0,\mu_1)$ be a coupling between them.

\begin{deff}[Class of path bridges]
\label{deff:path-bridge}
Let $\mathcal{F}$ be a class of absolutely continuous curves $\gamma : [0,1] \to M$. We say that $\mathcal{F}$ is a \emph{class of path bridges} if it satisfies the following properties:

\begin{enumerate}
\item \textbf{Endpoint connectivity:} for every pair $(x_0,x_1)\in M \times M$, there exists a path $\gamma^{x_0,x_1}\in \mathcal{F}$ such that
$
\gamma^{x_0,x_1}(0)=x_0, \gamma^{x_0,x_1}(1)=x_1.
$

\item \textbf{Measurable selection:} there exists a measurable map,
\begin{align*}
&\Gamma : M \times M \to \mathcal{F},\\
&(x_0,x_1)\mapsto \gamma^{x_0,x_1},
\end{align*}
such that for every $(x_0,x_1)\in M \times M$, the selected path $\gamma^{x_0,x_1}:[0,1] \to M$ belongs to $\mathcal{F}$ and satisfies
$
\gamma^{x_0,x_1}(0)=x_0,\gamma^{x_0,x_1}(1)=x_1,
$ where \(\mathcal F\) is understood with its natural measurable structure as a
class of paths.
\end{enumerate}
\end{deff}
\begin{rmk}
\label{rmk:path-bridge-chow-condition}
Let \((M,g)\) be a Riemannian manifold, and let
\(X_1,\ldots,X_k\in \mathfrak{X}(M)\) be smooth vector fields satisfying the
Chow--Rashevskii bracket-generating condition, as in
Theorem~\ref{thm:chow}. Then the Chow--Rashevskii theorem gives rise to a
natural class of path bridges associated with
\(X_1,\ldots,X_k\). Namely, one may take \(\mathcal F\) to be the class of
absolutely continuous curves \(\gamma:[0,1]\to M\) such that,
\[
    \dot\gamma(t)
    \in
    \operatorname{span}
    \{X_1(\gamma(t)),\ldots,X_k(\gamma(t))\},
\]
for almost every \(t\in[0,1]\). Under the bracket-generating assumption, any
two points in the same connected component of \(M\) can be connected by a curve
in \(\mathcal F\).
\end{rmk}
\begin{rmk}[Random interpolating path]
\label{rmk:random-interpolating-path}
Given a class of path-bridges $\mathcal{F}$, we can construct a random interpolation path between $\mu_0$ and $\mu_1$. Let $(Z_0,Z_1)$ be a random pair with law $\pi$, and define the random path,
$
Z_t := \gamma_t^{Z_0,Z_1} = \Gamma(Z_0,Z_1)(t), t \in [0,1].
$
For each $t \in [0,1]$, let $\mu_t$ denote the law of $Z_t$:
$
\mu_t := \mathrm{Law}(Z_t).
$
By construction,
$
\mu_t
=
(\mathrm{ev}_t \circ \Gamma)_\# \pi,
$
where $\mathrm{ev}_t(\gamma)=\gamma(t)$ is the evaluation map at time $t$. In particular,
$
\mu_0 = \mathrm{Law}(Z_0),
\mu_1 = \mathrm{Law}(Z_1),
$
so the family $(\mu_t)_{t \in [0,1]}$ is an interpolation between $\mu_0$ and $\mu_1$.
\end{rmk}
The following bridge-to-velocity construction is inspired by the
flow-matching viewpoint~\cite{lipman2022flow}: a random family of bridges
induces an Eulerian velocity field through conditional averaging. More
precisely, given a class of path bridges \(\mathcal F\), let
\((Z_0,Z_1)\) be a random pair with law \(\pi\), and define the random path,
\[
    Z_t:=\gamma_t^{Z_0,Z_1}
    =
    \Gamma(Z_0,Z_1)(t),
    \qquad t\in[0,1].
\]
The associated Eulerian velocity field is defined by,
\[
    v_t(x)
    =
    \mathbb E\!\left[
        \dot\gamma_t^{Z_0,Z_1}
        \,\middle|\,
        \gamma_t^{Z_0,Z_1}=x
    \right].
\]
When the path bridges are associated to a bracket-generating control
system (See Remark \ref{rmk:path-bridge-chow-condition}), the resulting Eulerian velocity admits a feedback-control
representation. Related
flow-matching and continuity-equation formulations for control-affine systems
have been studied in~\cite{elamvazhuthi2025flow,caponigro2025transport}. We use
this viewpoint to motivate the controlled-flow parameterization below.
\begin{thm}
\label{thm:general-matching}
Let $\mu_0$ and $\mu_1$ be probability measures on $(M,g)$, and let
$
\pi \in \Pi(\mu_0,\mu_1)
$
be a coupling. Let $\mathcal{F}$ be a class of path bridges as in Definition \ref{deff:path-bridge}.
Let $(Z_0,Z_1)$ be a random pair with law $\pi$, and define the random path
$
Z_t := \gamma_t^{Z_0,Z_1}, t\in[0,1]
$ (See Remark \ref{rmk:random-interpolating-path}).
Let
$
\mu_t := \mathrm{Law}(Z_t),
$
and $
v_t(x)
=
\mathbb{E}\big[\dot\gamma_t^{Z_0,Z_1}\mid \gamma_t^{Z_0,Z_1}=x\big].
$
Under standard integrability/regularity assumptions (see Section~\ref{sec:proofs} for details), the following statements hold.

\begin{enumerate}
\item[(i)] The pair $(\mu_t,v_t)$ satisfies the continuity equation
$
\partial_t \mu_t + \mathrm{div}_g(\mu_t v_t)=0
$
in the weak sense, that is, for every $\varphi \in C_c^\infty(M)$ and for almost every $t\in[0,1]$,
$$
\frac{d}{dt}\int_{M}\varphi(x)\,d\mu_t(x)
=
\int_{M}d \varphi(v_t(x))d\mu_t(x)
=\int_{M}g_x(\nabla_g\varphi(x), v_t(x))d\mu_t(x).$$
\item[(ii)] Consider the ODE,
\begin{align*}
  &  \dot Y_t = v_t(Y_t),\\
  & Y_0=y,
\end{align*}
Assume that the time-dependent vector field \(v_t\) is regular enough to generate a unique
non-autonomous flow map \(\phi_{0,t}\), and that the associated continuity equation is uniquely solved
by the corresponding flow pushforward. For simplicity, write
\(\phi_t:=\phi_{0,t}\), then,
$$
\mu_t = (\phi_t)_\# \mu_0
\qquad \text{for all } t\in[0,1].
$$
In particular,
$
(\phi_1)_\# \mu_0 = \mu_1.
$
\end{enumerate}
\end{thm}

The next corollary specializes the preceding bridge-to-velocity construction to
admissible bridges generated by a bracket-generating control system.
\begin{cor}[Pushforward of Distributions via Controlled Flows]
\label{cor:push}
Let \((M,g)\) be a Riemannian manifold, and let
\(X_1,\ldots,X_k\in\mathfrak{X}(M)\) be smooth vector fields satisfying the
Chow--Rashevskii bracket-generating condition. Assume that the hypotheses of Theorem~\ref{thm:general-matching} hold for the
class of path bridges associated with \(X_1,\ldots,X_k\). There exist measurable
feedback control laws \(u_i(t,x)\), \(1\le i\le k\), such that the associated
dynamical system,
\[
    \dot{x}(t)
    =
    \sum_{i=1}^k u_i(t,x(t))\,X_i(x(t)),
\]
generates a flow \(\phi_t\) satisfying,
\[
    (\phi_1)_\#\mu_0=\mu_1.
\]
In particular, the initial measure \(\mu_0\) can be transported to the target
measure \(\mu_1\) by feedback controls depending only on \((t,x)\), rather
than explicitly on the initial point.
\end{cor}
To operationalize this framework, we seek a set of vector fields whose cardinality is independent of the ambient dimension but still satisfies the bracket-generating condition in Theorem~\ref{thm:chow}. The following lemma shows that in $\mathbb{R}^d$, just two appropriately chosen vector fields suffice to generate the entire tangent space. Combined with Corollary~\ref{cor:push}, this result enables the design of expressive generative models using only two scalar control functions.

\begin{lem}
\label{lem:main}
On $\mathbb{R}^d$ (with coordinates $x_1, \dots, x_d$), $d\ge 3$, define the vector fields,
\[
V_1 = \frac{\partial}{\partial x_1} + x_2 \frac{\partial}{\partial x_3} + x_3 \frac{\partial}{\partial x_4} + \dots + x_{d-1} \frac{\partial}{\partial x_d}, \qquad V_2 = \frac{\partial}{\partial x_2}.
\]
Then the Lie algebra generated by $V_1$ and $V_2$ spans the tangent space $T_x\mathbb{R}^d$ at every point $x \in \mathbb{R}^d$.
\end{lem}

\begin{proof}
The proof is provided in Section \ref{sec:proofs}.
\end{proof}
\begin{rmk}[Coordinate permutations]
The construction in Lemma~\ref{lem:main} is not tied to the particular ordering
\((x_1,\ldots,x_d)\). For any permutation \((i_1,\ldots,i_d)\) of the coordinate
indices, the pair,
\[
    V_1
    =
    \partial_{x_{i_1}}
    +
    x_{i_2}\partial_{x_{i_3}}
    +
    \cdots
    +
    x_{i_{d-1}}\partial_{x_{i_d}},
    \qquad
    V_2=\partial_{x_{i_2}}
\]
is also bracket-generating. This follows either by repeating the bracket
calculation or by observing that bracket generation is invariant under
coordinate permutations.
\end{rmk}
\begin{rmk}
\label{rmk:chow}
Under the assumptions of Corollary~\ref{cor:push}, Lemma~\ref{lem:main}
shows that two appropriately chosen vector fields are sufficient to obtain a
bracket-generating structure in \(\mathbb R^d\). This generalizes classical flow-based generative models \citep{chen2018neural,papamakarios2021normalizing,lipman2022flow}, where $M = \mathbb{R}^d$ and the vector fields are taken to be the standard basis \( X_i = \frac{\partial}{\partial x_i} \). In contrast, our framework allows for a much smaller set of vector fields, provided they generate the full tangent space via Lie brackets. This flexibility is the core insight behind our parameter-efficient approach.
\end{rmk}
\begin{rmk}[Parameter efficiency]
\label{rmk:parameter-efficiency}
Lemma~\ref{lem:main} shows that, in principle, two appropriately chosen
vector fields can bracket-generate the full tangent space in arbitrary ambient
dimension \(d\). Consequently, the number of learned scalar control channels
can be kept fixed at \(k=2\), independently of \(d\). This should be contrasted with standard continuous-time flow parameterizations,
where the model typically outputs a full \(d\)-dimensional velocity vector. In
that case, the number of output channels scales linearly with the ambient
dimension. In our formulation, the network only outputs the scalar controls
multiplying the fixed vector fields. 
Thus, when each control channel is modeled by a neural network of comparable
size, the bracket-generating construction can substantially reduce the number
of learned output channels while retaining expressivity through Lie-bracket
generation. This statement concerns the number of scalar output channels; the
total number of neural-network parameters may still depend on the ambient
dimension through the input layer and hidden representation.
\end{rmk}
To illustrate how Lemma~\ref{lem:main} operates in low dimensions, we consider the Heisenberg group, a canonical example in geometric control theory that demonstrates how Lie brackets enable access to new directions beyond the original vector fields.

\begin{expl}[Heisenberg Group]
\label{example:heisenberg}
The \emph{(first) Heisenberg group} $\mathbb{H}$ is the manifold $\mathbb{R}^3$ equipped with the group law,
\[
(x_1, x_2, x_3) \cdot (y_1, y_2, y_3) = \left(x_1 + y_1,\, x_2 + y_2,\, x_3 + y_3 + x_1 y_2\right).
\]
A standard horizontal generating pair of left-invariant vector fields is,
\[
V_2 = \frac{\partial}{\partial x_2} + x_1 \frac{\partial}{\partial x_3}, \qquad
V_1 = \frac{\partial}{\partial x_1}.
\]
These fields satisfy,
\[
[V_1, V_2] = \frac{\partial}{\partial x_3},
\]
and together $\{V_1, V_2, [V_1, V_2]\}$ span the full tangent space $T_{x}\mathbb{R}^3$ at every point.
\end{expl}
\subsection{Learning Distribution Flows via Control Vector Fields}
Building on Lemma~\ref{lem:main}, we introduce a CNF-style generative
modeling algorithm that transports a simple source distribution
\(\mu_0\) (e.g., a standard Gaussian) to a target data distribution
\(\mu_1\) through a controlled continuous-time flow. The central idea is to
replace an unconstrained full-dimensional vector field by a structured control-weighted combination of fixed geometric vector fields with learned scalar
coefficients. Thus, instead of directly learning an arbitrary velocity
field \(v_\theta(t,x)\in\mathbb R^d\), we parameterize it as,
\[
    v_\theta(t,x)
    =
    \sum_{i=1}^k a_i(t,x;\theta)V_i(x),
\]
where \(V_1,\ldots,V_k\) are fixed vector fields and
\(a_i(t,x;\theta)\) are scalar control functions. When the \(V_i\) satisfy
the bracket-generating condition of Theorem~\ref{thm:chow}, the resulting
system retains strong controllability properties while using only a small
number of learned scalar control channels. In the Euclidean setting, we use the explicit construction from
Lemma~\ref{lem:main}. Namely, we define,
\[
    V_1
    =
    \frac{\partial}{\partial x_1}
    +
    x_2\frac{\partial}{\partial x_3}
    +\cdots+
    x_{d-1}\frac{\partial}{\partial x_d},
    \qquad
    V_i=\frac{\partial}{\partial x_i},
    \quad 2\le i\le k.
\]
For \(k=2\), Lemma~\ref{lem:main} shows that \(V_1,V_2\) already
bracket-generate \(T\mathbb R^d\). More generally, one may choose
\(k\ge 2\) fixed vector fields satisfying the same bracket-generating
condition.
\begin{algorithm}[t]
\caption{CNF-style ChowFlow Training in \(\mathbb R^d\)}
\label{alg:cnf-chowflow}
\begin{algorithmic}[1]
\State \textbf{Input:} number of control fields $2\le k \le d$, dimension $d\ge 3$;
samples $\{y^{(i)}\}_{i=1}^N\!\sim\!\mu_1$, base distribution $\mu_0$ with density \(p_0\), terminal time \(T=1\), number of ODE steps \(K\).
\State \textbf{Output:} Trained scalar controls \(a_i(t,x;\theta)\), $1\le i \le k$.

\State Initialize $a_i(\cdot;\theta_i)$ as MLPs with input $(t,x)$, $1\le i\le k$.
\State $V_1 \gets \partial_{x_1} + x_2 \partial_{x_3} + \cdots + x_{d-1}\partial_{x_d}$.
\State $V_i \gets \partial_{x_i}$ for $2 \le i \le k$.
\State Define the velocity field,
\[
    v_\theta(t,x)
    :=
    \sum_{i=1}^k a_i(t,x;\theta)V_i(x).
\]

\While{not converged}
    \State Sample minibatch \(\{y^{(j)}\}_{j=1}^B\sim \mu_{1}\)
    \For{each \(y^{(j)}\)}
        \State Integrate the dynamics backward from \(t=T\) to \(t=0\) using \(K\) ODE steps, while accumulating the divergence:
        \[
        \frac{d X^{(j)}_t}{dt}=v_\theta(t,X^{(j)}_t),
        \qquad
        X^{(j)}_T=y^{(j)}.
        \]
        \State Let \(z^{(j)}:=X^{(j)}_0\) and,
        \[
        \Delta^{(j)}:=\int_0^T \mathrm{div}(v_\theta(t,X^{(j)}_t))\,dt.
        \]
        \State Compute:
        \[
        \log p_\theta(y^{(j)})
        =
        \log p_0(z^{(j)})-\Delta^{(j)}.
        \]
    \EndFor
    \State Minimize the negative log-likelihood,
    \[
    \mathcal L_{\mathrm{CNF}}(\theta)
    =
    -\frac{1}{B}\sum_{j=1}^B
    \log p_\theta(y^{(j)}).
    \]
    \State Update \(\theta\) by backpropagation through the ODE solver.
\EndWhile

\State \Return trained controls \(a_i(t,x;\theta)\)
\end{algorithmic}
\end{algorithm}
\subsection{Parameterization.}
Each scalar control \(a_i\) is implemented as a neural network with input
\((t,x)\in [0,T]\times\mathbb R^d\) and scalar output. This is a
Markovian parameterization: unlike the endpoint-matching formulation, the
controls do not explicitly depend on the initial point \(x_0\). This
choice allows the model to define a standard continuous-time flow and to
be trained by likelihood using the continuous normalizing flow
change-of-variables formula.

\subsection{Dynamics.}
The learned flow is defined by the ODE,
\[
    \dot x(t)
    =
    v_\theta(t,x(t))
    =
    \sum_{i=1}^k a_i(t,x(t);\theta)V_i(x(t)).
\]
When \(V_i=\partial_{x_i}\) for \(i=1,\ldots,d\), this reduces to the
usual continuous normalizing flow parameterization,
\[
    v_\theta(t,x)
    =
    \bigl(a_1(t,x;\theta),\ldots,a_d(t,x;\theta)\bigr).
\]
\subsection{Likelihood training.}
Let \(p_0\) denote the density of the base distribution \(\mu_0\), for
example \(p_0\) may be the density of a standard Gaussian distribution. The density \(p_t\) induced by the ODE
satisfies the continuity equation,
\[
    \partial_t p_t+\mathrm{div}((p_t v_\theta))=0.
\]
Equivalently, along a trajectory \(x_t\) solving
\(\dot x_t=v_\theta(t,x_t)\), the log-density evolves according to,
\[
    \frac{d}{dt}\log p_t(x_t)
    =
    -\mathrm{div}(v_\theta(t,x_t)).
\]
Thus, for a data sample \(x_T\sim\mu_1\), we integrate the
dynamics backward from \(t=T\) to \(t=0\), while accumulating
the divergence correction:
\[
    \frac{d x_t}{dt}=v_\theta(t,x_t),
    \qquad
    x_T=x_{\mathrm{data}},\quad 
\]
If the backward integration gives \(z=x_0\) and,
\[
    \Delta
    =
    \int_0^T
    \mathrm{div}( v_\theta(t,x_t))\,dt,
\]
then the model likelihood is computed by,
\[
    \log p_\theta(x_{\mathrm{data}})
    =
    \log p_0(z)-\Delta.
\]
The training objective is the negative log-likelihood,
\[
    \mathcal L_{\mathrm{CNF}}(\theta)
    =
    -\mathbb E_{x\sim\mu_1}
    \left[
        \log p_\theta(x)
    \right].
\]

\subsection{Optimization.}
In practice, the augmented ODE is solved numerically, and the divergence
\(\mathrm{div} (v_\theta\)) is computed either exactly by automatic
differentiation or approximately using a trace estimator in high
dimensions. The parameters \(\theta\) of the scalar control networks are
trained by backpropagation through the ODE solver. Algorithm~\ref{alg:cnf-chowflow}
summarizes the resulting CNF-style ChowFlow training procedure.
\begin{figure}
    \centering
    \includegraphics[width=1.1\linewidth]{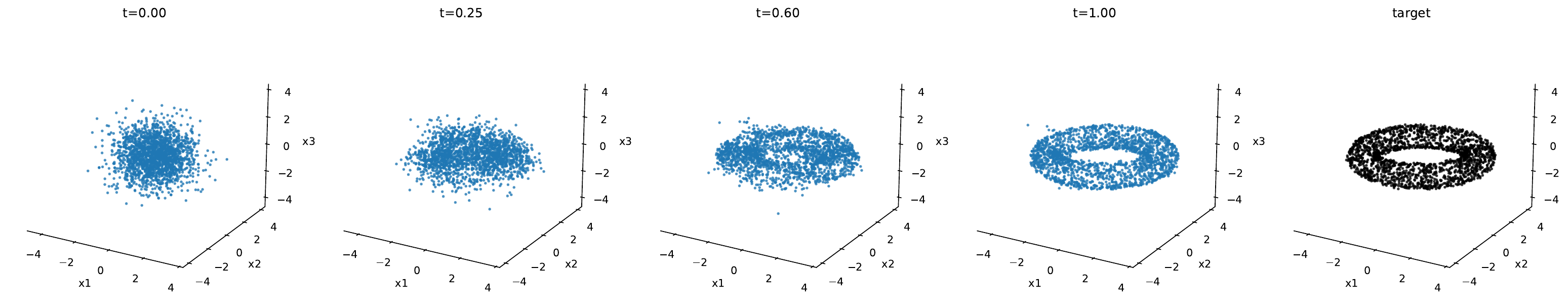}
    \includegraphics[width=1.1\linewidth]{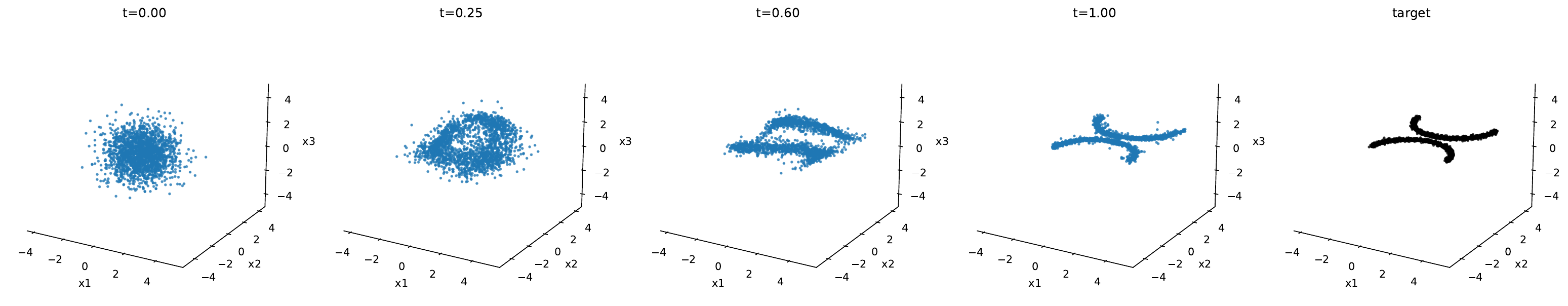}
    \includegraphics[width=1.1\linewidth]{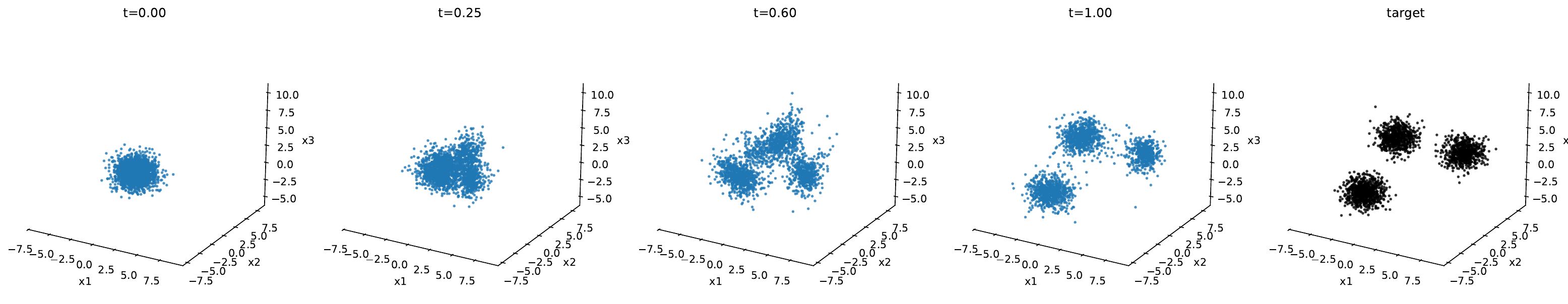}
    \caption{Algorithm \ref{alg:cnf-chowflow} applied to 3D synthetic distributions. Each row illustrates the learned transport from a simple Gaussian base distribution (left) to a complex target distribution (right). The target samples, shown in black, include a torus, two moons, and a Gaussian mixture. Intermediate transported samples are shown in blue. The flow qualitatively captures the geometric structure of the target using only two control channels (as constructed in Lemma~\ref{lem:main}), despite the ambient dimension being three.}
    \label{fig:synthetic}
\end{figure}

\section{Experiments}
\label{sec:experiments}

In this section, we evaluate the proposed framework through a series of experiments. We begin by demonstrating that our framework enables parameter-efficient generative modeling on several standard synthetic benchmarks. In particular, we consider three-dimensional settings where the underlying control system is governed by two vector fields constructed as in Lemma~\ref{lem:main}.

\subsection{Simulation Setup.}
In synthetic experiments, the source distribution $\mu_0$ is a standard Gaussian in $\mathbb{R}^3$, i.e., $\mu_0 = \mathcal{N}(0, I)$. The target distribution $\mu_1$ varies across tasks and includes several challenging benchmarks commonly used to evaluate flow-based models: (1) a two-moons distribution embedded in three dimensions, (2) a nonlinearly embedded torus (ring), and (3) a mixture of well-separated Gaussians. These targets exhibit either multimodality or complex geometry, making them nontrivial to model.

The system dynamics are defined by two fixed, smooth vector fields $V_1$ and $V_2$ (as described in Lemma~\ref{lem:main}), while the scalar control functions $a_1(t, x)$ and $a_2(t, x)$ are parameterized by small MLPs. Training is performed by maximum likelihood using the continuous normalizing flow change-of-variables formula. Specifically, we integrate data samples backward through
the learned controlled dynamics to the Gaussian base distribution while accumulating
the divergence term of the induced vector field. Additional training and hyperparameter
details are provided in Section~ \ref{sec:exper}.

\subsection{Visualization of Learned Transport.}
Figure~\ref{fig:synthetic} illustrates the transport induced by Algorithm \ref{alg:cnf-chowflow} on several synthetic distributions in $\mathbb{R}^3$, including a mixture of Gaussians, a two-moons manifold embedded in 3D, and a nonlinear torus-like structure. In each row, we visualize three key stages: (1) the initial samples drawn from a simple Gaussian prior (left), (2) intermediate samples during transport (middle, shown in blue), and (3) final transported samples (right). The target samples are visualized in black for reference.

Despite using only two control channels, the learned flows qualitatively match
the geometric structure of each target. In the Gaussian mixture case, the model separates and places mass near the correct modes. In the two-moons case, the flow deforms the isotropic prior into a curved bimodal support. For the torus, the learned flow captures the ring-like geometry without collapsing or overshooting. These results highlight the expressiveness of the proposed framework in low-dimensional settings. They also provide a geometric intuition for how our framework leverages the bracket structure of its vector fields to construct complex transformations with minimal parameterization.

\section{Proofs}
\label{sec:proofs}
In this section, we provide detailed proofs of the theoretical results stated above. We begin with Lemma \ref{lem:main}.
\begin{lem*}[Restatement of Lemma \ref{lem:main}]

On $\mathbb{R}^d$ (with coordinates $x_1, \dots, x_d$), $d\ge 3$, define the vector fields
\[
V_1 = \frac{\partial}{\partial x_1} + x_2 \frac{\partial}{\partial x_3} + x_3 \frac{\partial}{\partial x_4} + \dots + x_{d-1} \frac{\partial}{\partial x_d}, \qquad V_2 = \frac{\partial}{\partial x_2}.
\]
Then the Lie algebra generated by $V_1$ and $V_2$ spans the tangent space $T_x\mathbb{R}^d$ at every point $x \in \mathbb{R}^d$.
\end{lem*}

\begin{proof}
First, we compute the commutator:
\[
[V_1, V_2] = V_1 V_2 - V_2 V_1.
\]
Since $V_2 = \frac{\partial}{\partial x_2}$, only the term in $V_1$ involving $x_2$ contributes. Thus,
\[
[V_1, V_2] = -\frac{\partial}{\partial x_3}.
\]
We now claim that for $k = 1, \dots, d-2$,
\[
\operatorname{ad}_{V_1}^k(V_2) := [V_1, [V_1, \dots, [V_1, V_2] \dots ]] = (-1)^{k}\frac{\partial}{\partial x_{k+2}}.
\]
We prove this by induction on $k$. For the base case $k = 1$, we have already shown $[V_1, V_2] = -\frac{\partial}{\partial x_3}$. Suppose that for some $k \geq 1$,
\[
\operatorname{ad}_{V_1}^k(V_2) = (-1)^{k}\frac{\partial}{\partial x_{k+2}}.
\]
Then,
\[
[V_1, \frac{\partial}{\partial x_{k+2}}]
\]
acts nontrivially only on the term $x_{k+2} \frac{\partial}{\partial x_{k+3}}$ in $V_1$, whose derivative with respect to $x_{k+2}$ is exactly $\frac{\partial}{\partial x_{k+3}}$. Therefore,
\[
[V_1, \frac{\partial}{\partial x_{k+2}}] = -\frac{\partial}{\partial x_{k+3}}.
\]

By induction, the Lie algebra generated by $V_1$ and $V_2$ contains $\frac{\partial}{\partial x_i}$ for all $i = 2, \dots, d$ up to sign. In particular, for any $x=(x_{1},...,x_{d}) \in \R^{d}$, we have $\frac{\partial}{\partial x_{i}}$ for $2 \le i \le d$ in $T_{x}\R^{d}$. Together with $V_1$, the collection $\{V_1, \frac{\partial}{\partial x_2}, \dots, \frac{\partial}{\partial x_d}\}$ spans $T_{x}\R^{d}$ at every point $x$. Thus, the conditions of Theorem~\ref{thm:chow} are satisfied.
\end{proof}
\begin{expl}[Relation of Lemma \ref{lem:main} to Heisenberg Group]
Consider the case $d=3$, so that our space is $\mathbb{R}^3$ with coordinates $(x_1, x_2, x_3)$. Define the vector fields,
\[
V_2 = \frac{\partial}{\partial x_2} + x_1 \frac{\partial}{\partial x_3}, \qquad
V_1 = \frac{\partial}{\partial x_1}.
\]
These are precisely the standard generating vector fields for the (first) Heisenberg group $\mathbb{H}$. To see that $V_1$ and $V_2$ generate the entire tangent space at each point, compute their commutator:
\[
[V_1, V_2] = \frac{\partial}{\partial x_3},
\]
Therefore, the Lie algebra generated by $V_1$ and $V_2$ contains the full set $\left\{ \frac{\partial}{\partial x_1}, \frac{\partial}{\partial x_2}, \frac{\partial}{\partial x_3} \right\}$ at every point. This is analogous to the structure described in Lemma~\ref{lem:main}, up to a permutation of coordinates.

\end{expl}
Next, we restate and prove Theorem~\ref{thm:general-matching}. Let
\(\mu_0\) and \(\mu_1\) be probability measures on \((M,g)\), and let
\(\pi\in\Pi(\mu_0,\mu_1)\) be a coupling between them. We impose the following
integrability and regularity assumptions. The integrability assumptions justify
the use of the chain rule along the random paths and the interchange of
expectation and time integration. The regularity assumption ensures that the
Eulerian velocity field induced by conditional averaging generates a unique
well-posed flow. Throughout the following, we assume that,
\begin{enumerate}
    \item
    \[
        \mathbb E\int_0^1 \Norm{\dot Z_t}_g\,dt<\infty;
    \]
    \item there exists a Borel vector field,
    \[
        v:[0,1]\times M\to TM,
        \qquad
        v_t(x)\in T_xM,
    \]
    such that,
    \[
        v_t(Z_t)=\mathbb E[\dot Z_t\mid Z_t]
    \]
    almost surely for almost every \(t\in[0,1]\). Assume in addition that the induced velocity field \(v\) is regular
enough to generate a unique global flow. For example, when \(M\) is compact, the condition,
\[
    v\in L^1([0,1];C^1(M;TM)),
\]
is sufficient to generate a unique non-autonomous flow and to ensure that the
continuity equation is uniquely solved by the corresponding flow pushforward.
\end{enumerate}
Under these assumptions, the field $v_t$ transports the interpolating marginals $(\mu_t)$ in the sense of the continuity equation.
\begin{thm*}[Restatement of Theorem \ref{thm:general-matching}]
Let $\mu_0$ and $\mu_1$ be probability measures on $(M,g)$, and let
$
\pi \in \Pi(\mu_0,\mu_1)
$
be a coupling. Let $\mathcal{F}$ be a class of path bridges as in Definition \ref{deff:path-bridge}.
Let $(Z_0,Z_1)$ be a random pair with law $\pi$, and define the random path
$
Z_t := \gamma_t^{Z_0,Z_1}, t\in[0,1]
$ (See Remark \ref{rmk:random-interpolating-path}).
Let
$
\mu_t := \mathrm{Law}(Z_t),
$
and $
v_t(x)
=
\mathbb{E}\big[\dot\gamma_t^{Z_0,Z_1}\mid \gamma_t^{Z_0,Z_1}=x\big].
$
Under the above assumptions, the following statements hold.

\begin{enumerate}
\item[(i)] The pair $(\mu_t,v_t)$ satisfies the continuity equation
$
\partial_t \mu_t + \mathrm{div}_g(\mu_t v_t)=0
$
in the weak sense, that is, for every $\varphi \in C_c^\infty(M)$ and for almost every $t\in[0,1]$,
$$
\frac{d}{dt}\int_{M}\varphi(x)\,d\mu_t(x)
=
\int_{M}d \varphi(v_t(x))d\mu_t(x)
=\int_{M}g_x(\nabla_g\varphi(x), v_t(x))d\mu_t(x).$$
\item[(ii)] Consider the ODE,
\begin{align*}
  &  \dot Y_t = v_t(Y_t),\\
  & Y_0=y,
\end{align*}
Assume that the time-dependent vector field \(v_t\) is regular enough to
generate a unique non-autonomous flow map \(\phi_{0,t}\), and that the
associated continuity equation is uniquely solved by the corresponding flow
pushforward. For simplicity, write
\(\phi_t:=\phi_{0,t}\), then,
$$
\mu_t = (\phi_t)_\# \mu_0
\qquad \text{for all } t\in[0,1].
$$
In particular,
$
(\phi_1)_\# \mu_0 = \mu_1.
$
\end{enumerate}
\end{thm*}

\begin{proof}
\begin{enumerate}
    \item[(i)]
    Fix $\varphi \in C_c^\infty(M)$. Since $t\mapsto Z_t$ is absolutely continuous for almost every realization, the chain rule yields,
$$
\frac{d}{dt}\varphi(Z_t)
=
g(\nabla_g\varphi(Z_t), \dot Z_t),
$$
for almost every $t\in[0,1]$ and almost surely. 
Because \(\varphi\in C_c^\infty(M)\), its Riemannian gradient is bounded,
$
\|\nabla_g\varphi\|_{L^\infty}<\infty.
$
Hence, by Cauchy--Schwarz,
\[
\left|
g_{Z_t}(\nabla_g\varphi(Z_t),\dot Z_t)
\right|
\le
\|\nabla_g\varphi\|_{L^\infty}\Norm{\dot Z_t}_g.
\]
By the assumption,
\[
\mathbb E\int_0^1 \Norm{\dot Z_t}_g\,dt<\infty,
\]
we obtain,
\[
\mathbb E\int_0^1
\left|
g_{Z_t}(\nabla_g\varphi(Z_t),\dot Z_t)
\right|dt
<\infty.
\]

Consequently, by the fundamental theorem of calculus for absolutely
continuous functions,
\[
    \frac{d}{dt}\mathbb E[\varphi(Z_t)]
    =
    \mathbb E\left[
        g_{Z_t}\bigl(\nabla_g\varphi(Z_t),\dot Z_t\bigr)
    \right],
\]
for almost every \(t\in[0,1]\). Since \(Z_t\) has law \(\mu_t\), we have,
\[
    \mathbb E[\varphi(Z_t)]
    =
    \int_M \varphi(x)\,d\mu_t(x).
\]
Consequently,
\[
    \frac{d}{dt}
    \int_M \varphi(x)\,d\mu_t(x)
    =
    \mathbb E\left[
        g_{Z_t}\bigl(\nabla_g\varphi(Z_t),\dot Z_t\bigr)
    \right].
\]

Next, we rewrite the right-hand side as an Eulerian integral depending only on
the current position \(Z_t\). Since \(\nabla_g\varphi(Z_t)\) is measurable with
respect to the sigma-algebra generated by \(Z_t\), the defining property of
conditional expectation gives,
\[
    \mathbb E\left[
        g_{Z_t}\bigl(\nabla_g\varphi(Z_t),\dot Z_t\bigr)
    \right]
    =
    \mathbb E\left[
        g_{Z_t}\bigl(
            \nabla_g\varphi(Z_t),
            \mathbb E[\dot Z_t\mid Z_t]
        \bigr)
    \right].
\]
By definition,
\[
    \mathbb E[\dot Z_t\mid Z_t]=v_t(Z_t),
\]
almost surely. Therefore,
\[
    \mathbb E\left[
        g_{Z_t}\bigl(\nabla_g\varphi(Z_t),\dot Z_t\bigr)
    \right]
    =
    \mathbb E\left[
        g_{Z_t}\bigl(\nabla_g\varphi(Z_t),v_t(Z_t)\bigr)
    \right].
\]
Using again that \(Z_t\sim\mu_t\), this becomes,
\[
    \mathbb E\left[
        g_{Z_t}\bigl(\nabla_g\varphi(Z_t),v_t(Z_t)\bigr)
    \right]
    =
    \int_M
        g_x\bigl(\nabla_g\varphi(x),v_t(x)\bigr)
    \,d\mu_t(x).
\]
Combining the identities above, we obtain, for every
\(\varphi\in C_c^\infty(M)\) and for almost every \(t\in[0,1]\),
\[
    \frac{d}{dt}
    \int_M \varphi(x)\,d\mu_t(x)
    =
    \int_M
        g_x\bigl(\nabla_g\varphi(x),v_t(x)\bigr)
    \,d\mu_t(x).
\]
Since the Riemannian gradient \(\nabla_g\varphi\) is characterized by,
\[
    d\varphi_x(w)
    =
    g_x\bigl(\nabla_g\varphi(x),w\bigr),
    \qquad
    w\in T_xM,
\]
the right-hand side can equivalently be written as,
\[
    \int_M d\varphi_x(v_t(x))\,d\mu_t(x).
\]
Thus,
\[
    \frac{d}{dt}
    \int_M \varphi\,d\mu_t
    =
    \int_M d\varphi(v_t)\,d\mu_t.
\]

We now identify this identity with the weak formulation of the continuity
equation. Formally, if \(\mu_t\) has a smooth density \(\rho_t\) with respect
to the Riemannian volume measure \(d\mathrm{vol}_g\), so that,
\[
    d\mu_t=\rho_t\,d\mathrm{vol}_g,
\]
then the Riemannian continuity equation is,
\[
    \partial_t\rho_t+\operatorname{div}_g(\rho_t v_t)=0.
\]
Multiplying this equation by a test function \(\varphi\in C_c^\infty(M)\)
and integrating over \(M\), we get,
\[
    \int_M \varphi\,\partial_t\rho_t\,d\mathrm{vol}_g
    +
    \int_M \varphi\,\operatorname{div}_g(\rho_t v_t)\,d\mathrm{vol}_g
    =
    0.
\]
The first term is,
\[
    \int_M \varphi\,\partial_t\rho_t\,d\mathrm{vol}_g
    =
    \frac{d}{dt}\int_M \varphi\,\rho_t\,d\mathrm{vol}_g
    =
    \frac{d}{dt}\int_M \varphi\,d\mu_t.
\]
For the second term, using the Riemannian integration-by-parts formula and
the fact that \(\varphi\) has compact support, we obtain,
\[
    \int_M \varphi\,\operatorname{div}_g(\rho_t v_t)\,d\mathrm{vol}_g
    =
    -
    \int_M d\varphi(v_t)\,\rho_t\,d\mathrm{vol}_g
    =
    -
    \int_M d\varphi(v_t)\,d\mu_t.
\]
Therefore the weak form becomes,
\[
    \frac{d}{dt}\int_M \varphi\,d\mu_t
    =
    \int_M d\varphi(v_t)\,d\mu_t.
\]
This is precisely the identity derived above. 
Hence the curve of measures \((\mu_t)_{t\in[0,1]}\), together with the
Eulerian velocity field \(v_t\), satisfies the Riemannian continuity equation,
\[
    \partial_t\mu_t+\operatorname{div}_g(\mu_t v_t)=0,
\]
in the weak sense.
\item[(ii)]
Define,
\[
    \nu_t := (\phi_t)_\#\mu_0 .
\]
We claim that \((\nu_t)\) also satisfies the Riemannian continuity equation
with velocity field \(v_t\) and initial datum \(\mu_0\). Fix again
\(\varphi\in C_c^\infty(M)\). By definition of pushforward,
\[
    \int_M \varphi(x)\,d\nu_t(x)
    =
    \int_M \varphi(\phi_t(y))\,d\mu_0(y).
\]
Differentiate in \(t\). Since \(t\mapsto \phi_t(y)\) solves the ODE,
\[
    \frac{d}{dt}\phi_t(y)=v_t(\phi_t(y)),
\]
the chain rule on the manifold gives,
\[
    \frac{d}{dt}\varphi(\phi_t(y))
    =
    d\varphi_{\phi_t(y)}\bigl(v_t(\phi_t(y))\bigr),
\]
Under the assumed well-posedness and regularity of the flow, differentiation
under the integral sign is justified. Hence, for almost every \(t\),
\[
    \frac{d}{dt}
    \int_M \varphi(x)\,d\nu_t(x)
    =
    \int_M
    d\varphi_{\phi_t(y)}
    \bigl(v_t(\phi_t(y))\bigr)
    \,d\mu_0(y).
\]
Equivalently, pushing forward by \(\phi_t\), this becomes,
\[
    \frac{d}{dt}
    \int_M \varphi(x)\,d\nu_t(x)
    =
    \int_M
    d\varphi_x(v_t(x))
    \,d\nu_t(x).
\]
Thus \((\nu_t)\) is a weak solution of the Riemannian continuity equation,
\[
    \partial_t\nu_t+\operatorname{div}_g(\nu_t v_t)=0.
\]
At time \(t=0\), since \(\phi_0=\mathrm{Id}\), we have,
\[
    \nu_0
    =
    (\phi_0)_\#\mu_0
    =
    \mu_0.
\]
By part (i), the family \((\mu_t)\) is also a weak solution of the same
continuity equation with the same initial datum \(\mu_0\). By the assumed
uniqueness of weak solutions in the class under consideration, it follows that,
\[
    \mu_t
    =
    \nu_t
    =
    (\phi_t)_\#\mu_0
    \qquad
    \text{for all } t\in[0,1].
\]
Evaluating at \(t=1\), we obtain,
\[
    (\phi_1)_\#\mu_0
    =
    \mu_1.
\]
Thus the flow \(\phi_1\) transports \(\mu_0\) to \(\mu_1\). This proves part
(ii).
\end{enumerate}
\end{proof}
\begin{cor*}[Restatement of Corollary
\ref{cor:push}]
Let \((M,g)\) be a Riemannian manifold, and let
\(X_1,\ldots,X_k\in\mathfrak{X}(M)\) be smooth vector fields satisfying the
Chow--Rashevskii bracket-generating condition. Assume that the hypotheses of Theorem~\ref{thm:general-matching} hold for the
class of path bridges associated with \(X_1,\ldots,X_k\). There exist measurable
feedback control laws \(u_i(t,x)\), \(1\le i\le k\), such that the associated
dynamical system,
\[
    \dot{x}(t)
    =
    \sum_{i=1}^k u_i(t,x(t))\,X_i(x(t)),
\]
generates a flow \(\phi_t\) satisfying,
\[
    (\phi_1)_\#\mu_0=\mu_1.
\]
In particular, the initial measure \(\mu_0\) can be transported to the target
measure \(\mu_1\) by feedback controls depending only on \((t,x)\), rather
than explicitly on the initial point.
\end{cor*}
\begin{proof}
Let \((Z_0,Z_1)\sim\pi\), and let,
\[
    Z_t=\gamma_t^{Z_0,Z_1},
\]
be the random admissible bridge associated with the admissible path class
generated by \(X_1,\ldots,X_k\). Since the bridges are admissible, for almost
every \(t\in[0,1]\) there exist measurable pathwise controls,
\[
    \alpha_i^{Z_0,Z_1}(t),\qquad 1\le i\le k,
\]
such that,
\[
    \dot Z_t
    =
    \sum_{i=1}^k
    \alpha_i^{Z_0,Z_1}(t)X_i(Z_t).
\]
By Theorem~\ref{thm:general-matching}, the induced Eulerian velocity field is,
\[
    v_t(x)=\mathbb E[\dot Z_t\mid Z_t=x].
\]
Using the admissible representation of \(\dot Z_t\), we obtain,
\[
    v_t(x)
    =
    \mathbb E\left[
        \sum_{i=1}^k
        \alpha_i^{Z_0,Z_1}(t)X_i(Z_t)
        \,\middle|\,
        Z_t=x
    \right].
\]
Since \(X_i(Z_t)=X_i(x)\) after conditioning on \(Z_t=x\), this gives,
\[
    v_t(x)
    =
    \sum_{i=1}^k
    \mathbb E\left[
        \alpha_i^{Z_0,Z_1}(t)
        \,\middle|\,
        Z_t=x
    \right]X_i(x).
\]
Define the feedback controls,
\[
    u_i(t,x)
    :=
    \mathbb E\left[
        \alpha_i^{Z_0,Z_1}(t)
        \,\middle|\,
        Z_t=x
    \right].
\]
Then,
\[
    v_t(x)
    =
    \sum_{i=1}^k u_i(t,x)X_i(x).
\]
Therefore the ODE,
\[
    \dot x(t)=v_t(x(t)),
\]
can be written as the feedback-controlled system,
\[
    \dot x(t)
    =
    \sum_{i=1}^k u_i(t,x(t))X_i(x(t)).
\]
By Theorem~\ref{thm:general-matching}, under the stated regularity assumptions,
the time-dependent vector field \(v_t\) generates a flow \(\phi_t\) satisfying,
\[
    \mu_t=(\phi_t)_\#\mu_0.
\]
Evaluating at \(t=1\), and using \(\mu_1=\mathrm{Law}(Z_1)\), we obtain,
\[
    (\phi_1)_\#\mu_0=\mu_1.
\]
Thus the initial measure \(\mu_0\) is transported to the target measure
\(\mu_1\) by feedback controls depending only on \((t,x)\).
\end{proof}

\section{Experimental Details}
\label{sec:exper}

This section provides additional details about the synthetic datasets used in the experiments presented in this note.

\subsection{3D Moon Dataset.} 
We construct the synthetic \textit{3D moon} dataset by extending the classic 2D two-moons distribution into three dimensions. The first ``moon'' lies on the $x$--$y$ plane as a semicircular arc of unit radius centered at the origin, with $z = 0$. The second ``moon'' is placed in the parallel plane $z = 1$, translated along the $y$-axis and horizontally flipped to form a complementary arc. Gaussian noise with standard deviation $0.1$ is added to each point to introduce variability. The resulting dataset exhibits a simple yet nontrivial topology in $\mathbb{R}^3$, requiring models to capture both geometric structure and local noise.

\subsection{Multi-Component Gaussian Mixture.} 
We generate a synthetic multimodal distribution $\mu_1$ in $\mathbb{R}^D$ using a mixture of three anisotropic Gaussian components. Each component is defined by a distinct mean vector in the first three dimensions---specifically, $[6,3, 3]$, $[-2, -3, -2]$, and $[0, 0, 5]$---with the remaining dimensions set to zero. Each component uses a fixed covariance vector. The batch is split approximately evenly across components, and samples are generated by scaling standard Gaussian noise with the specified variances and shifting by the corresponding means. The samples are then shuffled to eliminate ordering artifacts. This setup results in a structured, high-dimensional distribution with multiple modes, providing a challenging testbed for generative models.

\subsection{3D Torus Dataset.} 
We construct a synthetic dataset by sampling points from a 3D torus embedded in $\mathbb{R}^3$, defined by a major radius $R = 3.0$ and a minor radius $r = 0.75$. For each sample, two angles $\theta, \phi \sim \mathcal{U}(0, 2\pi)$ are drawn, and the point is computed using the parametric equations:
\[
x = (R + r \cos\phi)\cos\theta,\quad 
y = (R + r \cos\phi)\sin\theta,\quad 
z = r \sin\phi.
\]
To simulate realistic conditions, Gaussian noise with standard deviation $0.07$ is independently added to each coordinate. The resulting point cloud lies on a tubular manifold, offering a challenging geometry with nontrivial curvature and topology.
\subsection{Control parameterization.}
The scalar controls \(a_i\) are parameterized by multilayer perceptrons. Each scalar control network is an MLP with three hidden layers of width \(128\)
and SiLU activations. We train for \(3000\) iterations with minibatch size
\(B=512\).
\subsection{Numerical integration.}
For the CNF-style experiments, the controls are Markovian functions of time and
state, \(a_i=a_i(t,x;\theta)\), and the induced velocity field is,
\[
    v_\theta(t,x)
    =
    \sum_{i=1}^k a_i(t,x;\theta)V_i(x).
\]
We simulate the dynamics over the time interval \([0,1]\). During training, for
each data sample \(x_1\sim \mu_1\), we integrate the controlled ODE backward
from \(t=1\) to \(t=0\),
\[
    \dot x_t = v_\theta(t,x_t),
    \qquad x_{t=1}=x_1.
\]
In the implementation, we use a fixed-step Runge--Kutta method with \(K\) steps.
For sampling, we draw \(z\sim p_0=\mathcal N(0,I)\) and integrate the same ODE
forward from \(t=0\) to \(t=1\) to obtain generated samples.
This is the discretization used in Algorithm~\ref{alg:cnf-chowflow}.

\subsection{Training objective.}
Training is performed using the continuous normalizing flow likelihood objective.
Along each trajectory, the log-density evolves according to the instantaneous
change-of-variables formula,
\[
    \frac{d}{dt}\log p_t(x_t)
    =
    -\mathrm{div} (v_\theta(t,x_t)).
\]
Given a minibatch of data samples \(\{x_1^{(j)}\}_{j=1}^B\sim\mu_1\), we
integrate each sample backward to obtain a latent point \(z^{(j)}=x_0^{(j)}\)
and simultaneously accumulate the divergence term. The model log-likelihood is
computed as,
\[
    \log p_\theta(x_1^{(j)})
    =
    \log p_0(z^{(j)})
    -
    \int_0^1
    \mathrm{div}(v_\theta(t,x_t^{(j)}))\,dt,
\]
where \(p_0=\mathcal N(0,I)\). We minimize the negative log-likelihood,
\[
    \mathcal L_{\mathrm{CNF}}(\theta)
    =
    -
    \frac1B
    \sum_{j=1}^B
    \log p_\theta(x_1^{(j)}).
\]
In the low-dimensional synthetic experiments, the divergence is
computed exactly using automatic differentiation. The training loss curves for the Gaussian-mixture and torus experiments are shown in Figure \ref{fig:synthetic-training-curves}.

\subsection{Optimization.}
\begin{figure}[t]
    \centering
    \includegraphics[width=0.45\linewidth]{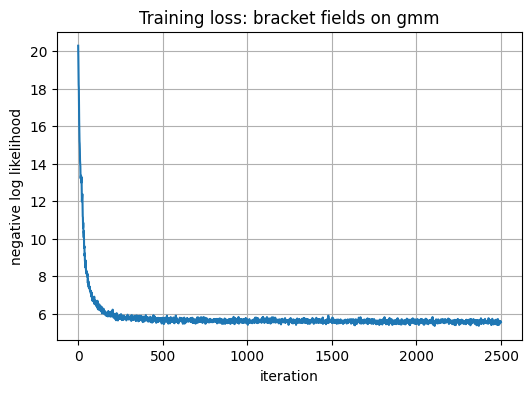}
    \includegraphics[width=0.45\linewidth]{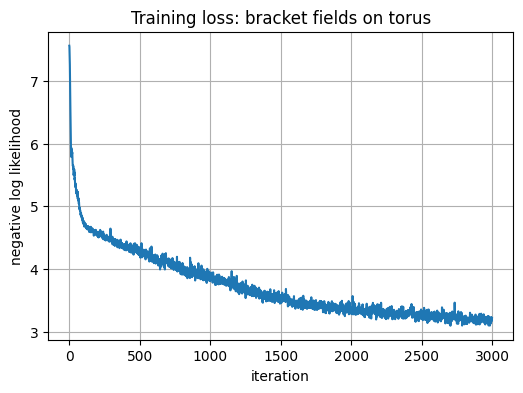}
    \caption{Negative log-likelihood training curves for the synthetic
    Gaussian-mixture and torus experiments.}
    \label{fig:synthetic-training-curves}
\end{figure}
We train using Adam with learning rate \(10^{-3}\). Gradients are backpropagated
through the discretized ODE solver and the accumulated divergence term. We clip
the gradient norm at \(10\) to improve numerical stability. We use a fixed integration horizon \(T=1\), \(K=16\) Runge--Kutta
steps during training, and a larger number of steps during sampling for improved
sample quality.

\clearpage  
\bibliographystyle{amsalpha} 
\bibliography{name.bib}
\end{document}